\documentclass{article}
\usepackage{kr}

\usepackage{times}
\usepackage{soul}
\usepackage{url}
\usepackage[hidelinks]{hyperref}
\usepackage{doi}
\usepackage{xurl}
\usepackage[utf8]{inputenc}
\usepackage[small]{caption}
\usepackage{graphicx}
\usepackage{amsmath}
\usepackage{amsthm}
\usepackage{booktabs}
\usepackage{algorithm}
\usepackage{algorithmic}
\usepackage{amssymb}
\usepackage{enumitem}
\usepackage[T1]{fontenc}

\newtheorem{theorem}{Theorem}
\newtheorem{definition}{Definition}
\newtheorem{proposition}{Proposition}
\newtheorem{observation}{Observation}

\newcommand{\AF}{\mathsf{AF}}

\newcommand{\CDAF}{\mathsf{CDAF}}
\newcommand{\Args}{\mathcal{A}}
\newcommand{\attacks}{\rightarrow}
\newcommand{\Contexts}{\mathcal{C}}

\newcommand{\pref}{\mathsf{pref}}
\newcommand{\grounded}{\mathsf{gr}}
\newcommand{\stable}{\mathsf{stb}}
\newcommand{\complete}{\mathsf{comp}}
\newcommand{\Persps}{\Pi}

\title{Choosing the Lens:\\Strategic Perspective Activation in Context-Dependent Argumentation}

\author{
Albert Sadowski$^1$\and
Jarosław A. Chudziak$^1$\\
\affiliations
$^1$Warsaw University of Technology\\
\emails
\{albert.sadowski.stud, jaroslaw.chudziak\}@pw.edu.pl
}

\begin{document}

\maketitle

\begin{abstract}
The same arguments often need to be evaluated under different external regimes. An agent with influence over the regime has a strategic lever that standard formalisms do not directly capture. We introduce context-dependent argumentation frameworks (CDAFs), an extension of Dung's theory in which a defeat function determines, per context, which attacks succeed. Blocked attacks are inverted rather than deleted, so extensions stay conflict-free with respect to the attack relation. A perspective-labeled specialisation derives the defeat function from a relevance set $\rho$ and a priority $\pi$. The relevance set is the agent's action space. In a small worked example, the agent's target argument is rejected under every full-relevance priority, yet accepted under a partial activation whose outcome no VAF audience can mirror. We define the corresponding decision problem, ACTIVATION-MANIPULATION, and record baseline complexity bounds. For grounded semantics with mandatory perspectives the problem is NP-complete, and the hardness comes from the activation choice itself.
\end{abstract}

\section{Introduction}
\label{sec:intro}

Consider a team proposing an architectural change to a production service. The proposal accumulates arguments about latency, operational complexity, deployment cost, and regression risk. These claims are not in dispute as facts; what is in dispute is how they bear on the decision. To a performance reviewer, latency dominates and complexity is secondary. A reliability reviewer reads the same arguments through incident risk, and is unmoved by latency improvements without stability evidence. The finance reviewer cares about deployment and headcount cost. The arguments and their conflicts are fixed across reviewers. Which attacks succeed depends in part on which lens governs the review.

Suppose the sponsor has some influence over which lens applies. They may choose which review tracks to pursue, who joins early design discussions, and which forums hear the proposal first. They cannot change how each reviewer ranks considerations within a lens; the reliability reviewer always prioritises stability. But some set of lenses ends up governing the review. We call this the regime. Whoever decides which lenses are foregrounded has a strategic lever.

Standard argumentation~\cite{dung1995} has no parameter for the regime. One can evaluate a separate Dung framework for each candidate regime, but then the choice between regimes lives outside the theory. Value-based argumentation~\cite{benchCapon2003} comes closer: each audience induces its own defeat pattern, and acceptance under some or every audience is a standard question there. But an audience is an ordering of the value set. It can rerank values, not deactivate them; in particular, an attack between two arguments that share a value succeeds for every audience. The regime we study varies the roster of active perspectives, not their ranking, and this lever is not expressible as an audience.

We propose an extension of Dung's framework, context-dependent argumentation frameworks (CDAFs), and use it to study activation manipulation: can the agent choose the regime so that their target argument is accepted? We work through an example in which the target is reachable only via partial-relevance activation. We formalise it as the decision problem ACTIVATION-MANIPULATION and record baseline complexity bounds. For grounded semantics with mandatory perspectives we prove NP-completeness; this bound is due to the activation choice, not to the base semantics. A fuller study of the framework is beyond this paper's scope; here we focus on the strategic angle.

\section{Context-Dependent Argumentation Frameworks}
\label{sec:cdaf}

\begin{definition}[Context-dependent argumentation framework]
\label{def:cdaf}
A \emph{context-dependent argumentation framework} is a tuple $\CDAF = \langle \Args, R, \Contexts, \delta \rangle$ where $\Args$ is a finite set of arguments, $R \subseteq \Args \times \Args$ is the attack relation, $\Contexts$ is a finite set of contexts, and $\delta \colon \Contexts \times R \to \{0,1\}$ is the \emph{defeat function}. We say that an attack $(a,b) \in R$ \emph{succeeds} in $c$ if $\delta(c,(a,b)) = 1$, and that it is \emph{blocked} otherwise. For each context $c \in \Contexts$, the \emph{induced framework} is $\AF_c = \langle \Args, D_c \rangle$ with
\begin{equation*}
\begin{split}
D_c = {} & \{(a,b) \in R : \delta(c,(a,b)) = 1\} \\
& \cup \{(b,a) : (a,b) \in R,\ \delta(c,(a,b)) = 0\}.
\end{split}
\end{equation*}
\end{definition}

A successful attack is kept as a defeat. A blocked attack is inverted: the conflict is resolved in favour of the attacked argument. Since each $\AF_c$ is a standard Dung framework, the classical semantics carry over per context: a $\sigma$-extension in $c$ is just a $\sigma$-extension of $\AF_c$~\cite{baroni2018}, for any of $\sigma \in \{\grounded, \pref, \stable, \complete\}$. We write $\sigma(c)$ for the set of $\sigma$-extensions of $\AF_c$.

\begin{proposition}
\label{prop:conflictfree}
For every context $c$, every set that is conflict-free in $\AF_c$ is conflict-free with respect to $R$. In particular, every $\sigma$-extension in $c$ is conflict-free with respect to $R$.
\end{proposition}

\textit{Proof.} Let $(x, y) \in R$. If the attack succeeds in $c$, then $(x, y) \in D_c$; if it is blocked, then $(y, x) \in D_c$. Either way, no conflict-free set of $\AF_c$ contains both $x$ and $y$. $\qed$

\paragraph{Blocked attacks.}
The simpler alternative is to delete blocked attacks, evaluating $\langle \Args, \{(a,b) \in R : \delta(c,(a,b)) = 1\} \rangle$. This is what a VAF audience does: defeats form a subset of the attacks~\cite{benchCapon2003}. Deletion has a known defect. Extensions can contain both ends of an attack in $R$. This violates conflict-freeness with respect to $R$; at the structured level it yields inconsistent conclusions, against the consistency postulates~\cite{caminadaAmgoud2007}. In the example below, deletion would accept $\{a, d, t\}$ while $(a,t) \in R$. Preference-based argumentation faced the same defect, and inverting the blocked attack is one of its standard repairs~\cite{amgoud2014}. ASPIC$^+$ instead redefines conflict-freeness over attacks rather than defeats, so a blocked attack still bars co-acceptance~\cite{PrakkenModgil2013}. Extended argumentation frameworks defend such co-acceptance outright when the attack is asymmetric~\cite{modgil2009}. We adopt inversion. With deletion, Proposition~\ref{prop:conflictfree} does not hold. The intended reading: an attack the context blocks is an objection that is not voiced, and an unvoiced objection is resolved against the silenced party.

\paragraph{Perspective-labeled CDAFs.}
In the specialisation below, the defeat function is derived rather than given. Every argument carries a source perspective. A context activates a subset of perspectives; a priority ranks the perspectives. Defeat then follows from these assignments.

\begin{definition}[Perspective-labeled CDAF]
\label{def:plcdaf}
A \emph{perspective-labeled CDAF} is a tuple $\langle \Args, R, \Contexts, \Persps, \mathit{src}, \rho, \pi \rangle$ with $\Persps$ a finite set of perspectives, $\mathit{src} \colon \Args \to \Persps$ assigning each argument to its source, $\rho \colon \Contexts \to 2^{\Persps}$ giving the active perspectives in each context, and $\pi \colon \Contexts \times \Persps \to \mathbb{N}$ a priority function. The defeat function is
\begin{equation*}
\begin{split}
\delta_\pi(c, (a,b)) = 1 \;\;\text{iff}\;\; & \mathit{src}(a) \in \rho(c) \\
& \wedge \pi(c, \mathit{src}(a)) \geq \pi(c, \mathit{src}(b)).
\end{split}
\end{equation*}
\end{definition}

When $|\Contexts| = 1$ we drop the context argument and write $\rho \subseteq \Persps$, $\pi \colon \Persps \to \mathbb{N}$, and $D$ for the induced defeat relation. The example and the decision problem below use this form.

\paragraph{Comparison to VAFs and the action space.}
Perspectives are loosely analogous to values, and $(\rho, \pi)$ to a VAF audience. The differences are that $\pi$ may rerank perspectives across contexts, and that $\rho$ may deactivate them outright. A further difference is the handling of blocked attacks: a VAF audience deletes them, while we invert them. Under a deletion reading, a full-relevance context with an injective priority would be exactly an audience; inversion breaks the correspondence. The deactivation move is what drives the strategic story of this paper. We read $\rho$ as the agent's action space: the lens through which they choose to have arguments evaluated. The priority $\pi$ is not the agent's to change: it is part of the framework, fixed before the agent moves.

\paragraph{Design choices.}
The defeat condition tests the perspective of the attacker, while the priority of the attacked argument acts as a shield. Deactivating a perspective therefore costs its arguments their offensive force, not their standing. Two nearby designs fail for our purpose. If activation gated both endpoints of an attack, every attack on an argument would require that argument's perspective to be active, so an agent could shield any argument by muting its own perspective. That move would trivialise manipulation. If deactivation instead removed the arguments of inactive perspectives, it would change what is on the record rather than which objections carry force. That is a different phenomenon, and it has its own formalism: selecting which designated arguments are present is the control move of control argumentation frameworks~\cite{Dimopoulos2018}. It would also break the example below, where the winning move deactivates the perspective of the target itself; removal would delete the target rather than help it. Arguments of an inactive perspective must stay evaluable. The handling of blocked attacks is the remaining design choice, discussed above.

\section{A Worked Example}
\label{sec:example}

\paragraph{Setup.}
Let $\Args = \{a, b, t, d\}$, $\Persps = \{\alpha, \beta, \gamma\}$, $\mathit{src}(a) = \mathit{src}(t) = \alpha$, $\mathit{src}(b) = \beta$, $\mathit{src}(d) = \gamma$, and $R = \{(a,t), (b,t), (a,b), (b,a), (d,b)\}$. Table~\ref{tab:setup} summarises the structure.

\begin{table}[h]
\centering
\small
\begin{tabular}{cccc}
\toprule
Argument & Perspective & Attacked by & Attacks \\
\midrule
$a$ & $\alpha$ & $b$ & $t, b$ \\
$b$ & $\beta$ & $a, d$ & $t, a$ \\
$t$ & $\alpha$ & $a, b$ & --- \\
$d$ & $\gamma$ & --- & $b$ \\
\bottomrule
\end{tabular}
\caption{Argument structure of the worked example.}
\label{tab:setup}
\end{table}

The point of the construction is that $a$ and $t$ share a perspective, so the attack $(a, t)$ is intra-perspective. The agent's goal is to have $t$ accepted under the preferred semantics.

\paragraph{Result 1: $t$ is rejected under full relevance.}
For every priority $\pi \colon \Persps \to \mathbb{N}$, $t$ is not credulously preferred-accepted in the framework induced by $\rho = \Persps$.

\textit{Proof.} With $\rho = \Persps$, the attack $(a, t)$ succeeds under every $\pi$, since $\pi(\alpha) \geq \pi(\alpha)$ holds trivially. The only defeat that can target $a$ is $(b, a)$: no other attack in $R$ points at $a$, and no blocked attack inverts into one, because $(a, t)$ never fails and the inversion of $(a, b)$ is $(b, a)$ itself. Moreover, $(b, a) \in D$ iff $\pi(\beta) \geq \pi(\alpha)$, either as a success or as the inversion of a failed $(a, b)$. Now suppose $E$ is admissible and contains $t$. Then $E$ must defend $t$ against $(a, t)$. If $\pi(\beta) < \pi(\alpha)$, then $a$ has no defeater and the defense is impossible. If $\pi(\beta) \geq \pi(\alpha)$, the defense forces $b \in E$; but then $(b, t)$ succeeds as well, since $\mathit{src}(t) = \mathit{src}(a) = \alpha$, and $\{b, t\} \subseteq E$ violates conflict-freeness. Either way, no admissible set contains $t$. Injectivity is not needed: only the comparison between $\pi(\beta)$ and $\pi(\alpha)$ matters, ties included. The priority of $\gamma$ is irrelevant, since $(d, b)$ and its inversion touch only $b$ and $d$. $\qed$

\paragraph{Result 2: $t$ is accepted under $\rho = \{\beta, \gamma\}$.}
With $\pi(\alpha) = 1$, $\pi(\beta) = 2$, $\pi(\gamma) = 3$, and $\rho = \{\beta, \gamma\}$, the induced framework has the unique preferred extension $\{d, t\}$, so $t$ is credulously, in fact skeptically, preferred-accepted.

\textit{Proof.} Computing $\delta_\pi$ under $\rho = \{\beta, \gamma\}$ gives Table~\ref{tab:active}. The attacks $(b, t)$, $(b, a)$, and $(d, b)$ succeed. The attacks $(a, t)$ and $(a, b)$ fail, since $\alpha \notin \rho$, and contribute the inverted defeats $(t, a)$ and $(b, a)$. So $D = \{(b, t), (b, a), (d, b), (t, a)\}$.

\begin{table}[h]
\centering
\small
\begin{tabular}{ccccc}
\toprule
Attack & $\mathit{src} \in \rho$? & Priority & Succeeds? & Contributes \\
\midrule
$(a, t)$ & $\alpha \notin \rho$ & --- & no & $(t, a)$ \\
$(b, t)$ & $\beta \in \rho$ & $2 \geq 1$ & yes & $(b, t)$ \\
$(a, b)$ & $\alpha \notin \rho$ & --- & no & $(b, a)$ \\
$(b, a)$ & $\beta \in \rho$ & $2 \geq 1$ & yes & $(b, a)$ \\
$(d, b)$ & $\gamma \in \rho$ & $3 \geq 2$ & yes & $(d, b)$ \\
\bottomrule
\end{tabular}
\caption{Defeats induced under $\rho = \{\beta, \gamma\}$. A failed attack contributes its inversion.}
\label{tab:active}
\end{table}

The argument $d$ has no defeater, so it enters the grounded extension and defeats $b$. With $b$ defeated, the only defeater of $t$ is gone, so $t$ enters as well, and $t$ defeats $a$ through the inverted $(t, a)$. The grounded extension is $\{d, t\}$. Every preferred extension is complete and therefore contains it, and neither $a$ nor $b$ can be added, since $(t, a)$ and $(d, b)$ are defeats. So $\{d, t\}$ is the unique preferred extension, $t$ is accepted, and $a$ is rejected. $\qed$

Under any full-relevance evaluation, $t$ is caught in a structural trap. Its same-perspective neighbour $a$ attacks it, and the only argument that can defend $t$ against $a$ is $b$, which itself attacks $t$. Whenever $b$ is strong enough to defeat $a$, it is also strong enough to defeat $t$. The defender doubles as an attacker, and the trap holds for every priority. The agent's way out is to deactivate $\alpha$. Setting $\rho = \{\beta, \gamma\}$ blocks every attack from an $\alpha$-perspective argument, including the friendly-fire attack $(a, t)$. The move has a cost: the blocked attacks are resolved against $a$, so $a$ itself is rejected. The record does not show $a$ standing beside $t$; it shows $a$ overruled. The trap is broken: $d$, unaffected by the move since $\gamma \in \rho$, defeats $b$ and clears the only remaining threat to $t$. Together, Results 1 and 2 give the general point: rejection under every full-relevance priority can coexist with acceptance under a partial activation.

The example also separates activation from VAF audiences at the level of outcomes. Under the winning move $\rho = \{\beta, \gamma\}$, the unique preferred extension is $\{d, t\}$: the outcome accepts $t$ and rejects both $a$ and $b$. No VAF audience can produce this acceptance pattern, under any assignment of values.

\begin{observation}
\label{obs:non-vaf}
Let $\langle \Args, R \rangle$ be as in the setup. For every value assignment $\mathit{val} \colon \Args \to V$ and every audience, that is, every strict partial order on $V$, the resulting VAF credulously accepts $a$ or $b$ under preferred semantics. So no VAF over $\langle \Args, R \rangle$ reproduces the acceptance pattern of $\rho = \{\beta, \gamma\}$, which rejects both.
\end{observation}

\textit{Proof.} In a VAF, an attack $(x, y) \in R$ defeats unless $\mathit{val}(y) \succ \mathit{val}(x)$, so defeats form a subset of $R$~\cite{benchCapon2003}. Nothing in $R$ attacks $d$, and the only attack on $a$ in $R$ is $(b, a)$. Four cases cover all audiences. If $(b, a)$ does not defeat, then $a$ is unattacked and $\{a\}$ is admissible. If $(b, a)$ and $(a, b)$ both defeat, which happens exactly when neither value is preferred to the other, then $\{a\}$ is admissible, since $a$ counter-attacks its only attacker. If $(b, a)$ defeats and neither $(a, b)$ nor $(d, b)$ does, then $b$ is unattacked and $\{b\}$ is admissible. If $(b, a)$ and $(d, b)$ defeat and $(a, b)$ does not, then $\{a, d\}$ is admissible: $d$ is unattacked, $d$ defends $a$ against $b$, and $R$ has no attack between $a$ and $d$. In each case $a$ or $b$ belongs to some admissible set, hence to some preferred extension. Under $\rho = \{\beta, \gamma\}$, the unique preferred extension is $\{d, t\}$, so neither $a$ nor $b$ is accepted. $\qed$

The driving fact is that VAF defeats never leave $R$. Rejecting $a$ therefore requires an accepted attacker of $a$ inside $R$, and the only candidate is $b$; the four cases show that this constraint always saves $a$ or $b$. Inversion escapes the constraint because a blocked attack yields a defeat outside $R$. Here the blocked attack $(a, t)$ turns into a defeat by $t$, an argument that attacks nothing in $R$ and so defeats nothing in any VAF. The separation at the level of outcomes is specific to the inversion reading. Under deletion, the winning outcome would be $\{a, d, t\}$, and an audience can reach exactly that outcome by giving the four arguments distinct values ordered $v_t \succ v_d \succ v_a \succ v_b$. Under inversion, the outcome $\{d, t\}$ is out of reach for every audience.

\section{Activation Manipulation}
\label{sec:manipulation}

The strategic question of the previous section generalises to a decision problem. Given a perspective-labeled framework, a fixed priority, and a target argument, can the agent choose $\rho$ so that the target is accepted? We add a set of mandatory perspectives to the input. Mandatory perspectives model review tracks that the institution always convenes. They also keep the problem from degenerating: without them, $\rho = \emptyset$ is available, which under deletion would discard every attack and accept every argument outright, and under inversion would reverse every attack. A blocked attack becomes a defeat against its own source, so muting a perspective exposes its arguments to their own inverted attacks. A target that sources any attack cannot be isolated by silencing everything.

\medskip
\noindent\textsc{Activation-Manipulation}$_\sigma$.\\
\textit{Input.} A finite tuple $\langle \Args, R, \Persps, \mathit{src} \rangle$, a priority $\pi \colon \Persps \to \mathbb{N}$, a mandatory set $M \subseteq \Persps$, and a target argument $t \in \Args$.\\
\textit{Question.} Does there exist $\rho$ with $M \subseteq \rho \subseteq \Persps$ such that $t$ is credulously $\sigma$-accepted in the framework induced by $(\rho, \pi)$?
\medskip

For the upper bound we use a standard guess-and-check. Nondeterministically choose $\rho$ with $M \subseteq \rho \subseteq \Persps$, build the induced defeat relation in polynomial time, then verify $\sigma$-acceptance in the induced framework. For grounded semantics the verification is polynomial. For stable semantics, a stable extension containing $t$ can be guessed alongside $\rho$. The preferred case looks harder at first, since verifying that a set is a preferred extension requires checking maximality, which is co-NP. The escape is a standard fact: credulous preferred-acceptance coincides with credulous admissibility, because every admissible set extends to some preferred extension. So an admissible witness containing $t$ can be guessed alongside $\rho$. This places \textsc{Activation-Manipulation}$_\sigma$ in NP for $\sigma \in \{\grounded, \stable, \pref\}$.

For lower bounds, first set $\Persps = M = \{p_0\}$ with all arguments mapped to the single perspective and $\pi$ constant. The only allowed activation is $\rho = \{p_0\}$, every attack succeeds, and the induced framework is $\langle \Args, R \rangle$ itself, under both the inversion and the deletion reading. The problem restricted to these instances is exactly credulous $\sigma$-acceptance in standard Dung frameworks. This yields NP-hardness for $\sigma \in \{\stable, \pref\}$~\cite{baroni2018,dung1995}, so \textsc{Activation-Manipulation}$_\sigma$ is NP-complete in those cases. The activation choice is forced here, so these bounds are inherited from the base semantics and say nothing about the activation lever. Credulous grounded acceptance is in P, so the same restriction gives only P-hardness for grounded. The next result shows that the activation choice alone carries NP-hardness for grounded semantics.

\begin{theorem}
\label{thm:grounded}
\textsc{Activation-Manipulation}$_{\grounded}$ is NP-complete. Hardness holds already for a constant priority, and under both the inversion and the deletion reading of blocked attacks.
\end{theorem}

\textit{Proof.} Membership was shown above. For hardness we reduce from 3SAT. Let $\psi$ have variables $v_1, \dots, v_n$ and clauses $c_1, \dots, c_m$. Build an argument $x_l$ for every literal $l$, an argument $y_j$ for every clause $c_j$, and a target $\varphi$. The attacks are $x_l \attacks x_{\lnot l}$ for every literal, $x_l \attacks y_j$ whenever $l$ occurs in $c_j$, and $y_j \attacks \varphi$ for every $j$. Every literal argument $x_l$ gets its own perspective $p_l$, while the clause arguments and $\varphi$ share one perspective $p_C$. The priority $\pi$ is constant, so an attack succeeds iff the perspective of its source is in $\rho$. Set $M = \{p_C\}$ and the target to $\varphi$.

Suppose an assignment satisfies $\psi$. Take $\rho = \{p_C\} \cup \{p_l : l \text{ is true}\}$, so for each variable exactly one literal perspective is active. Each true-literal argument $x_l$ is then unattacked: the counter-attack $(x_{\lnot l}, x_l)$ fails, and no failed attack inverts into $x_l$, because none of the attacks sourced at $x_l$ fails. So the true literals enter the grounded extension. Every clause contains a true literal, whose attack on $y_j$ succeeds, so every $y_j$ is defeated. The attacks $(y_j, \varphi)$ all succeed, since $p_C \in \rho$, so $\varphi$ is defended and enters the grounded extension.

Conversely, suppose $\varphi$ is in the grounded extension $G$ for some $\rho \supseteq \{p_C\}$. All attacks $(y_j, \varphi)$ succeed, so $G$ defeats every $y_j$. The only defeats on $y_j$ come from literal arguments $x_l$ with $l \in c_j$ and $p_l \in \rho$. Under deletion this is immediate. Under inversion, nothing else can point at $y_j$, since the only attack sourced at $y_j$ is $(y_j, \varphi)$ and it never fails. So for every clause $c_j$ there is some $x_l \in G$ with $l \in c_j$ and $p_l \in \rho$. For such an $x_l$, the complementary perspective $p_{\lnot l}$ cannot be in $\rho$. Otherwise $(x_{\lnot l}, x_l)$ succeeds, and the only defeat on $x_{\lnot l}$ is $(x_l, x_{\lnot l})$ itself, since all attacks sourced at $x_{\lnot l}$ succeed and nothing inverts into it; the grounded fixpoint could then never defend $x_l$, so $x_l \notin G$. Now let $T = \{l : x_l \in G \text{ and } p_l \in \rho\}$. By the previous step, $T$ contains no complementary pair, and every clause contains a literal of $T$. Any assignment that makes $T$ true satisfies $\psi$. The same computation goes through under deletion. There the extraction must use $T$ rather than membership in $G$ alone: if both perspectives of a variable are inactive, deletion leaves both of its literal arguments unattacked and in $G$, but their attacks on clause arguments are gone, so they contribute nothing. $\qed$

Since credulous grounded acceptance in a fixed framework is polynomial, this hardness comes from the activation choice itself, not from the base semantics.

\section{Discussion}
\label{sec:discussion}

The strategic action we study here, varying which attacks constitute defeats while keeping the argument structure fixed, has three close neighbours. Control argumentation frameworks~\cite{Dimopoulos2018} give the agent a set of designated control arguments and ask whether including some subset makes a target accepted in every completion of an uncertain framework. The control move edits which arguments are present; our regime keeps the arguments and attacks fixed and varies only which attacks count as defeats. Incomplete argumentation frameworks~\cite{Baumeister2021} are the second neighbour, and the relation can be made precise. A choice of $\rho$ induces one of at most $2^{|\Persps|}$ defeat graphs, since all attacks sourced in one perspective toggle together after the priority filter. \textsc{Activation-Manipulation} is therefore possible credulous acceptance over a perspective-generated family of completions. The completions live in $R \cup R^{-1}$, the space that the direction-uncertain attacks of control frameworks range over. But incomplete and control frameworks read a completion as epistemic uncertainty about which framework is real; activation reads it as control over which framework is in force. The parallel extends to complexity. Possible credulous grounded acceptance in incomplete frameworks is NP-complete, already when only attacks are uncertain~\cite{Baumeister2021}, and Theorem~\ref{thm:grounded} shows that the hardness survives the compression of the completion family from $2^{|R|}$ to $2^{|\Persps|}$. Value-based argumentation~\cite{benchCapon2003} is the closest formal kin, and its subjective acceptance already asks whether some audience accepts a target. The lever differs: an audience reranks values while $\rho$ deactivates perspectives, and an audience deletes a blocked attack while we invert it. The example separates the two at the level of outcomes: activation reaches an acceptance pattern, the rejection of both $a$ and $b$, that no audience over any value assignment reaches (Observation~\ref{obs:non-vaf}). Whether every audience outcome is reachable by some activation is not settled here. Other extensions of Dung's theory modulate defeat through attacks on attacks~\cite{modgil2009} or per-node acceptance conditions~\cite{brewka2010}, but in each case the modulation is endogenous or fixed at definition time, and none treats activation as an action available to the agent.

The strategic question is not only theoretical. Multi-agent and LLM-based systems put argumentation to work in applied settings: dialectical refinement among agents for argument-component classification~\cite{baba2026}, and agent memory where goal-conditioned encodings of the same experience are reconciled by argumentation at retrieval~\cite{rashomon2026}. Such systems fix neither an attack relation nor a defeat function explicitly. But which perspectives are consulted at evaluation time is a design parameter there, and a version of activation manipulation arises once that parameter is under an agent's control.

Several questions remain open. The complexity of \textsc{Activation-Manipulation} without mandatory perspectives is unsettled for grounded semantics under inversion. Activations may carry costs, and the priority may be partly under the agent's control; both variants raise mechanism-design questions. The skeptical version, asking whether some $\rho$ accepts the target in every $\sigma$-extension, is also open. The most direct extension is the multi-agent setting, where several agents pick disjoint or overlapping subsets of $\Persps$ and the realised $\rho$ is formed by union or intersection. This yields cooperative and adversarial variants, and the existence of strategic equilibria under such dynamics is, we think, the most productive direction for further study. The framework deserves a fuller treatment. Here we record one strategic angle and a worked example that supports it.

\bibliographystyle{kr}
\bibliography{references}

\end{document}